\def\year{2021}\relax
\documentclass[letterpaper]{article} 
\usepackage{xcolor}
\usepackage{aaai21}  
\usepackage{times}  
\usepackage{helvet} 
\usepackage{courier}  
\usepackage[hyphens]{url}  
\usepackage{graphicx} 
\usepackage{natbib}  
\usepackage{caption} 
\newtheorem{proof}{Proof}
\newtheorem{proposition}{Proposition}
\newtheorem{corollary}{Corollary}
\usepackage{comment}

\usepackage{graphicx}
\usepackage{subcaption}

\usepackage{amsmath}
\usepackage{amssymb}

\usepackage[ruled,vlined,linesnumbered,longend]{algorithm2e}
\SetKwInput{KwInput}{Input}
\SetKwInput{KwOutput}{Output}

\usepackage{multirow} 
\usepackage{adjustbox} 
\usepackage{booktabs} 

\title{Counterfactual Fairness with \\
Disentangled Causal Effect Variational Autoencoder}
\author {
    \normalsize{
        Hyemi Kim\textsuperscript{\rm 1},
        Seungjae Shin\textsuperscript{\rm 1},
        JoonHo Jang\textsuperscript{\rm 1},
        Kyungwoo Song\textsuperscript{\rm 1},
        Weonyoung Joo\textsuperscript{\rm 1},
        Wanmo Kang\textsuperscript{\rm 2},
        Il-Chul Moon\textsuperscript{\rm 1}}\\
}
\affiliations {
    \textsuperscript{\rm 1} Department of Industrial and Systems Engineering \\
    \textsuperscript{\rm 2} Department of Mathematical Sciences \\
    Korea Advanced Institute of Science and Technology (KAIST), Daejeon, Korea\\
    \small{\texttt{\{khm0308,tmdwo0910,adkto8093,gtshs2,es345,wanmo.kang,icmoon\}@kaist.ac.kr}}}
\begin{document}

\maketitle

\begin{abstract}
The problem of fair classification can be mollified if we develop a method to remove the embedded sensitive information from the classification features. This line of separating the sensitive information is developed through the causal inference, and the causal inference enables the counterfactual generations to contrast the what-if case of the opposite sensitive attribute. Along with this separation with the causality, a frequent assumption in the deep latent causal model defines a single latent variable to absorb the entire exogenous uncertainty of the causal graph. However, we claim that such structure cannot distinguish the 1) information caused by the intervention (i.e.,  sensitive variable) and 2) information correlated with the intervention from the data. Therefore, this paper proposes Disentangled Causal Effect Variational Autoencoder (DCEVAE) to resolve this limitation by disentangling the exogenous uncertainty into two latent variables: either 1) independent to interventions or 2) correlated to interventions without causality. Particularly, our disentangling approach preserves the latent variable correlated to interventions in generating counterfactual examples. We show that our method estimates the total effect and the counterfactual effect without a complete causal graph. By adding a fairness regularization, DCEVAE generates a counterfactual fair dataset while losing less original information. Also, DCEVAE generates natural counterfactual images by only flipping sensitive information. Additionally, we theoretically show the differences in the covariance structures of DCEVAE and prior works from the perspective of the latent disentanglement.
\end{abstract}

\section{Introduction}
\begin{figure}[t]
  \centering
  \includegraphics[width=0.99\linewidth]{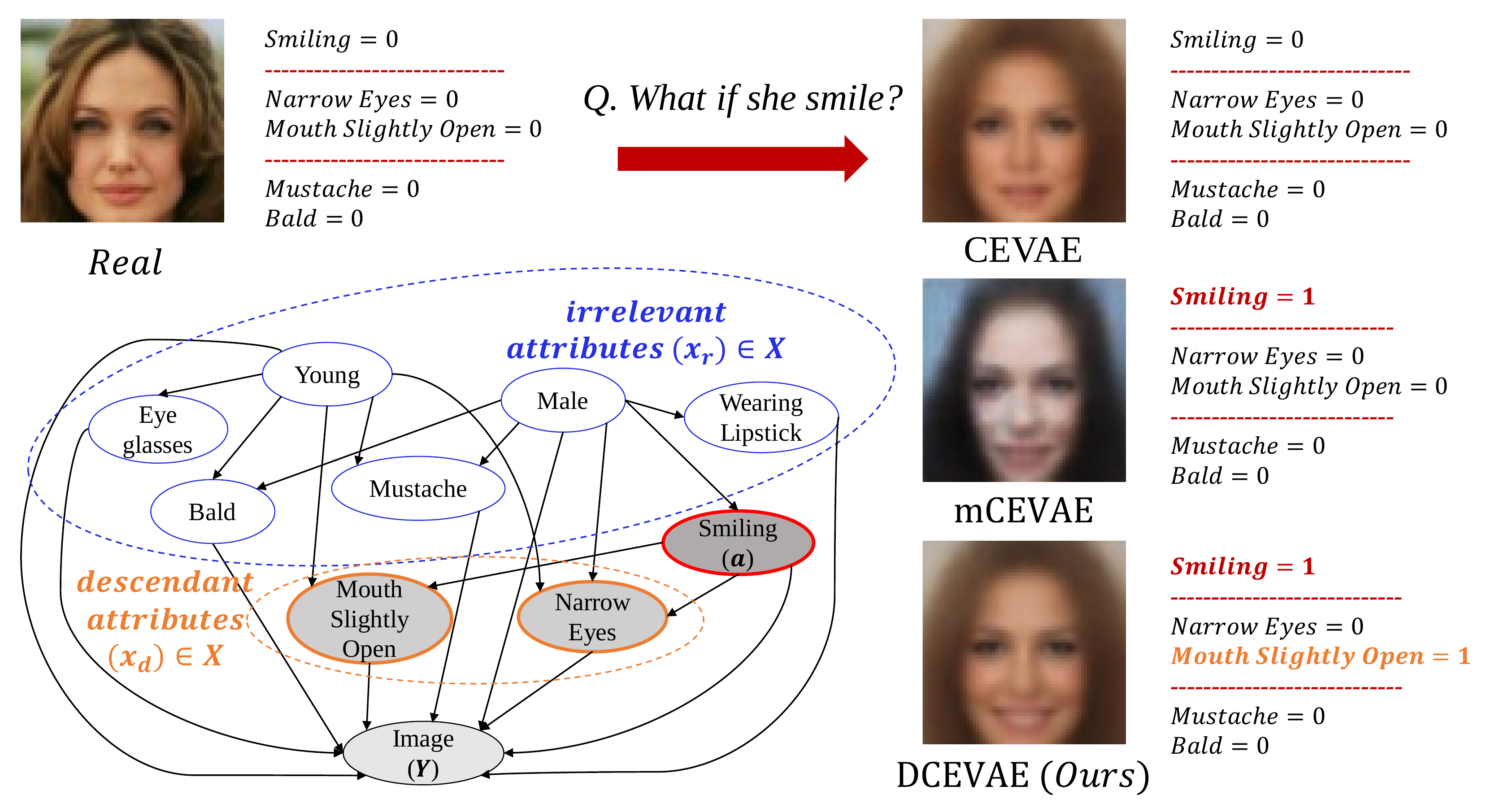}
  \caption{Counterfactual examples have a set of information attributes that are either 1) maintained or 2) altered when the intervention variable, $a$, is altered. For example, a counterfactually generated image for $a_{\textit{Smiling}}=0$ should be labeled as $a_{\textit{Smiling}}=1$, and 
such change may cause the subsequent change on descendant attributes of $a$, $x_d$ (i.e. \textit{Mouth Slightly Open}, \textit{Narrow Eyes}) by maintaining the other attributes intact, $x_r$. Prior works of CEVAE and mCEVAE fail to maintain the irrelevant attributes of $a$, while DCEVAE shows a counterfactual image with the irrelevant attributes of $a$, undisturbed.}
  \label{fig:overall}
\end{figure}
\noindent
Machine learning has penetrated our lives so deep, and its fairness and societal utilization have become a growing concern in our society \cite{aleo2008foreclosure, kim2019multiaccuracy}. The incident of \textit{COMPAS} \cite{brennan2009evaluating} shows that the learning model can be a source of unfairness in our judicial system by discriminating people by race. Given that the learners become unfair only because of training data \cite{hardt2016equality}, we ask the question of whether it is feasible to correct its unfairness from the data or not. Particularly, our concept of unfairness comes from our societal principle on equal treatments across races, gender, religion, etc., a.k.a. \textit{sensitive variables}, without prejudice. Then, the key question on the machine learning research becomes whether we can separate such prejudices embedded in the data algorithmically or not.

Considering the prejudice by the sensitive variable, the causal inference is an interesting tool to separate the factors contributing to decision-making. The objective of causal inference is learning the causal effect of an intervention variable, $a$, on individual features, $x$, and an outcome, $y$. Here, if we regard the intervention variable in causality as the sensitive variable in decision-making, the learning fairness can be formulated as the causal inference task \cite{zhang2018causal, chiappa2019path,wu2019counterfactual, kilbertus2017avoiding}. For example, a causal model estimates the effect of sensitive variables, such as race and gender, on an admission result \cite{kusner2017counterfactual}. Another study shows a causal model predicting a medication's effect on a patient's prognosis \cite{pfohl2019counterfactual}. If we focus on modeling the exogenous uncertainty with Variational Autoencoder (VAE) \cite{louizos2017causal, pfohl2019counterfactual}, it has been a common practice to introduce a single latent variable to reflect all exogenous uncertainty.

We separate different causal effects into multiple latent variables, so the diverse aspects of an intervention, features, and an outcome can be related in complex causal graphs. Subsequently, this separation of causal effects by factors enables complex counterfactual example generations because we can only intervene in the sensitive variables by leaving other variables intact. This counterfactual generation becomes our barometer in how fair a learning model is. If a model is fair, the model should result in the same classification for both original and counterfactual instances with an altered sensitive variable.

This paper starts by claiming the limitation of modeling the exogenous uncertainty with a single latent variable, and this paper develops a disentangling structure, or Disentamgled Causal Effect VAE (DCEVAE), for counterfactual generations to relax the limitation. Unlike the previous approaches with a single latent variable to model all features \cite{shalit2017estimating, louizos2017causal, pfohl2019counterfactual}, DCEVAE separates the latent variable to model the exogenous uncertainties either from the intervention or from the feature without the intervention.

As DCEVAE disentangles the uncertainty into two latent variables, DCEVAE has more accurate estimation performances on the total effect and the counterfactual effect compared to Causal VAE models with a single latent. DCEVAE added counterfactual fairness regularization to generate counterfactual fair examples with less transformation on the original dataset. Also, DCEVAE generates counterfactual images that do not naturally occur in the dataset, i.e., women with Mustache, through interventions. Finally, we analyze DCEVAE structure from the perspective of linear VAE, and we show DCEVAE is structured to separate the posterior covariance of the sensitive and the feature exogenous uncertainties.

\section{Preliminaries}

\subsection{Counterfactual Fairness Problem Formulation}
The final goal of this paper is to provide a counterfactual fair classification method through the latent disentanglement. From this aspect, we start our formulation from the definition of \textit{fairness}. We define $A$ as the sensitive attributes of an individual, which should not be used for discriminative tasks; $X$ as the other observed attributes of individuals; $Y$ as the dependent variable to estimate; and $\hat{Y}$ as the model estimation. \cite{kusner2017counterfactual} suggests the definition of counterfactual fairness and its relation to a causal graph.

A causal graph specifies $\mathcal{M}=\langle \mathbf{U},\mathbf{V},\mathbf{F},\mathbf{P}(u) \rangle$; and $\mathbf{V}$ is the set of endogenous variables, $P(v):=P(V=v)=\sum_{\{u|f_V(V,u)=v\}} P(u)$; and $\mathbf{U}$ is the set of exogenous variables, i.e. the stochastic elements of a variable; and $\mathbf{F}$ is the set of deterministic functions, $V_i=f_{V_i}(PA_{V_i}, U_{V_i})$ with indicating the parents of $V_i$ as $PA_{V_i}$ in a causal graph. With a causal graph, Eq. \ref{eq:counterfactual_fairness} defines the counterfactual fairness. 
\begin{equation}
\begin{aligned}
\label{eq:counterfactual_fairness}
P\big(\hat{y}_{A\leftarrow a}(U)&=y|X=x,A=a\big)\\
&=P\big(\hat{y}_{A\leftarrow \neg a}(U)=y|X=x,A=a\big)\\
\end{aligned}
\end{equation}
for all $y$ and any value $\neg a$ attainable by $A$.
Here, $U$ is the set of exogenous variables, and $\hat{Y}$ becomes two different variations by either $a$ or $\neg a$. The counterfactual fairness asserts that the estimated distributions on $\hat{Y}$s should be identical regardless of the sensitive value, $a$.

\begin{figure*}
\centering
\tabskip=0pt
\valign{#\cr
  \hbox{
    \begin{subfigure}{.12\textwidth}
    \centering
    \includegraphics[width=0.9\textwidth]{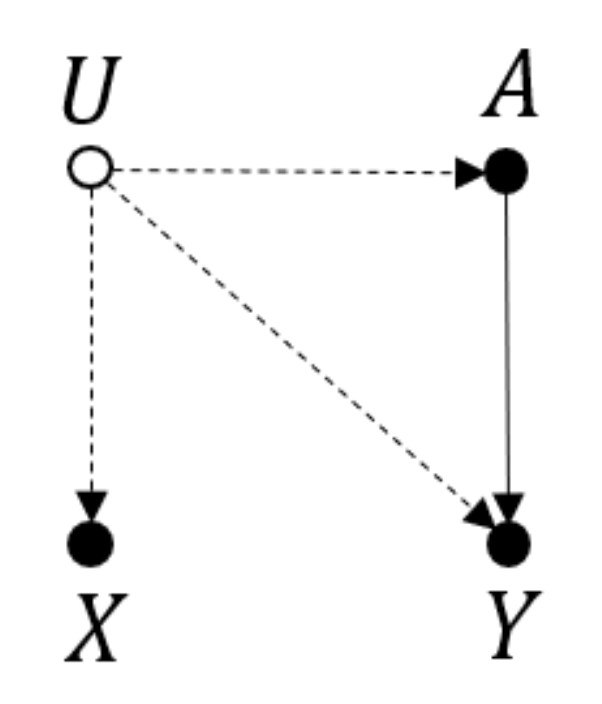}
    \caption{}
    \end{subfigure}
  }\vfill

  \hbox{
    \begin{subfigure}{.12\textwidth}
    \centering
    \includegraphics[width=0.8\textwidth]{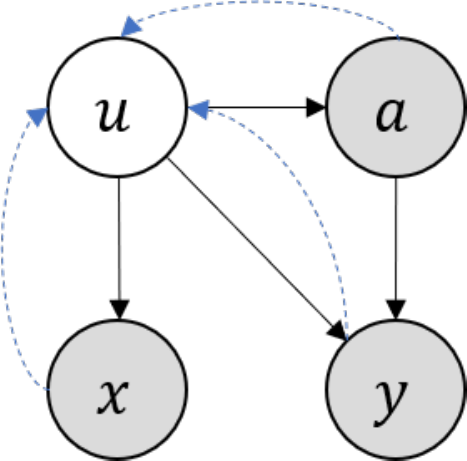}
    \caption{}
    \end{subfigure}
  }\cr
\hbox{
    \begin{subfigure}{.12\textwidth}
    \centering
    \includegraphics[width=0.9\textwidth]{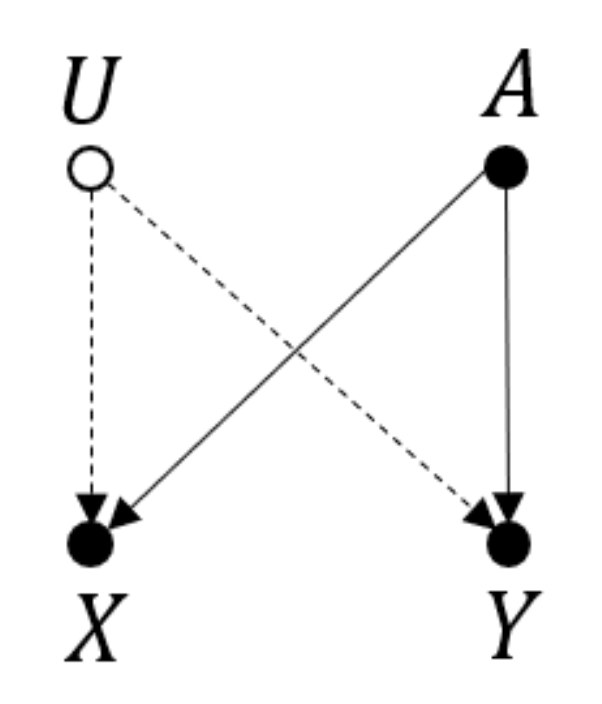}
    \caption{}
    \end{subfigure}
  }\vfill
  \hbox{
    \begin{subfigure}{.12\textwidth}
    \centering
    \includegraphics[width=0.8\textwidth]{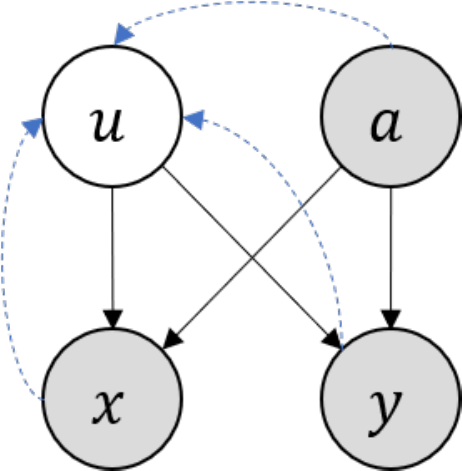}
    \caption{}
    \end{subfigure}
  }\cr
\hbox{
    \begin{subfigure}{.23\textwidth}
    \centering
    \includegraphics[width=0.8\textwidth]{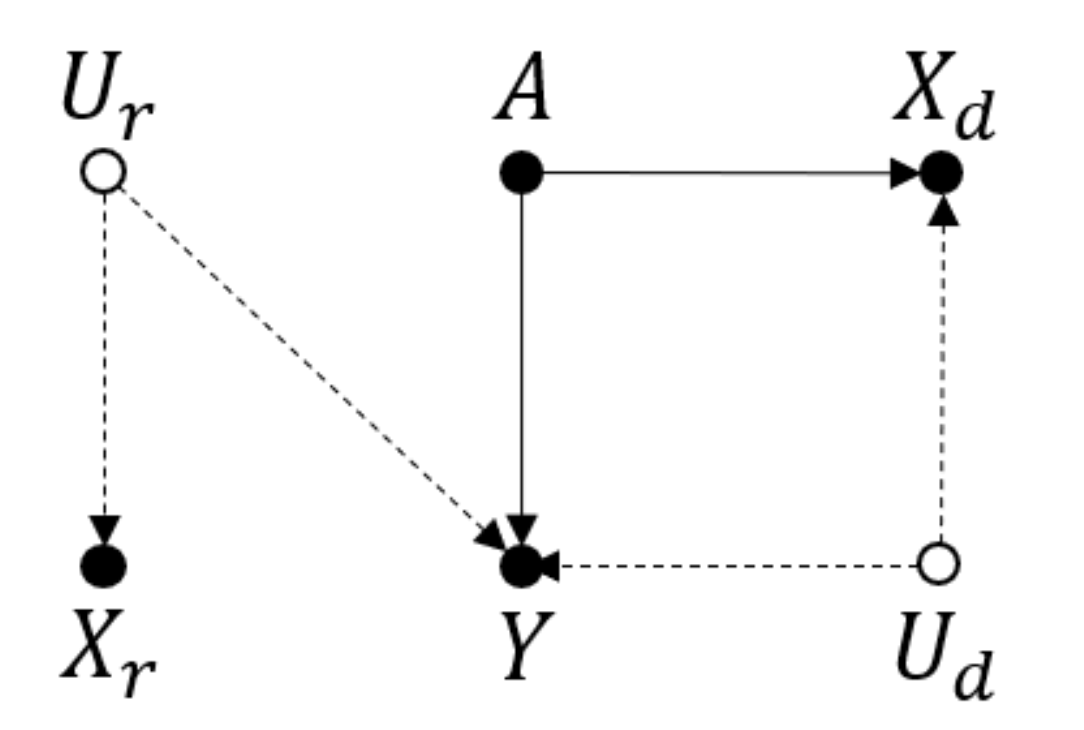}
    \caption{}
	\label{fig:mine_causal}
    \end{subfigure}
  }\vfill
  \hbox{
    \begin{subfigure}{.23\textwidth}
    \centering
    \includegraphics[height=1.9cm, width=0.75\textwidth]{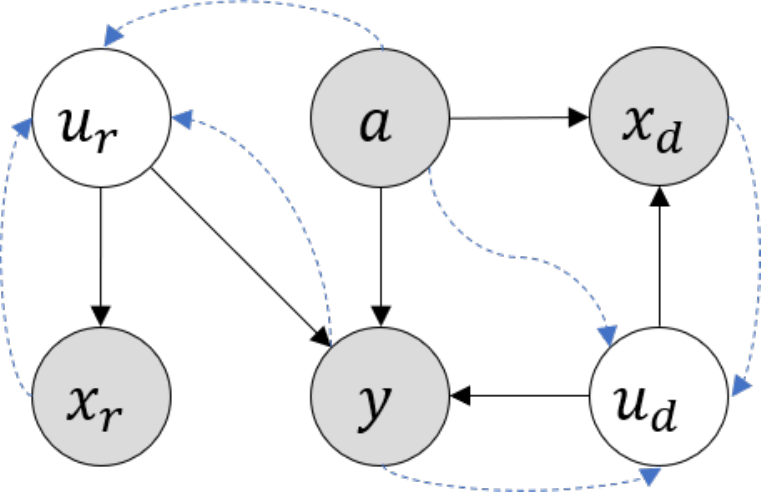}
    \caption{}
	\label{fig:mine_bayes}
    \end{subfigure}
  }\cr
	\hbox{
    \begin{subfigure}[b]{.25\textwidth}
    \centering
    \includegraphics[height=4.5cm,width=0.85\textwidth]{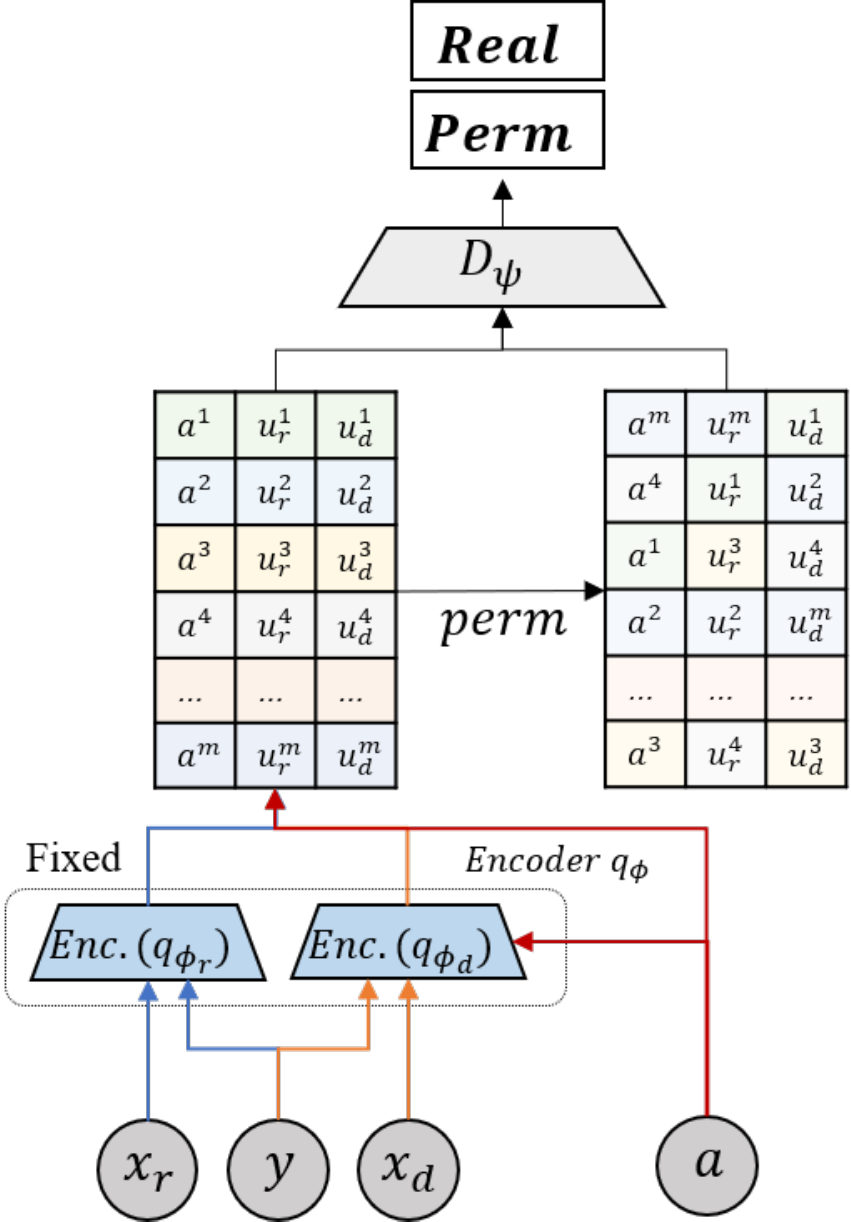}
    \caption{Neural Network (in max phase)}
    \end{subfigure}
	\label{fig:mine_train_max}
  }\cr
\hbox{
    \begin{subfigure}[b]{.3\textwidth}
    \centering
    \includegraphics[height=4.5cm,width=0.85\textwidth]{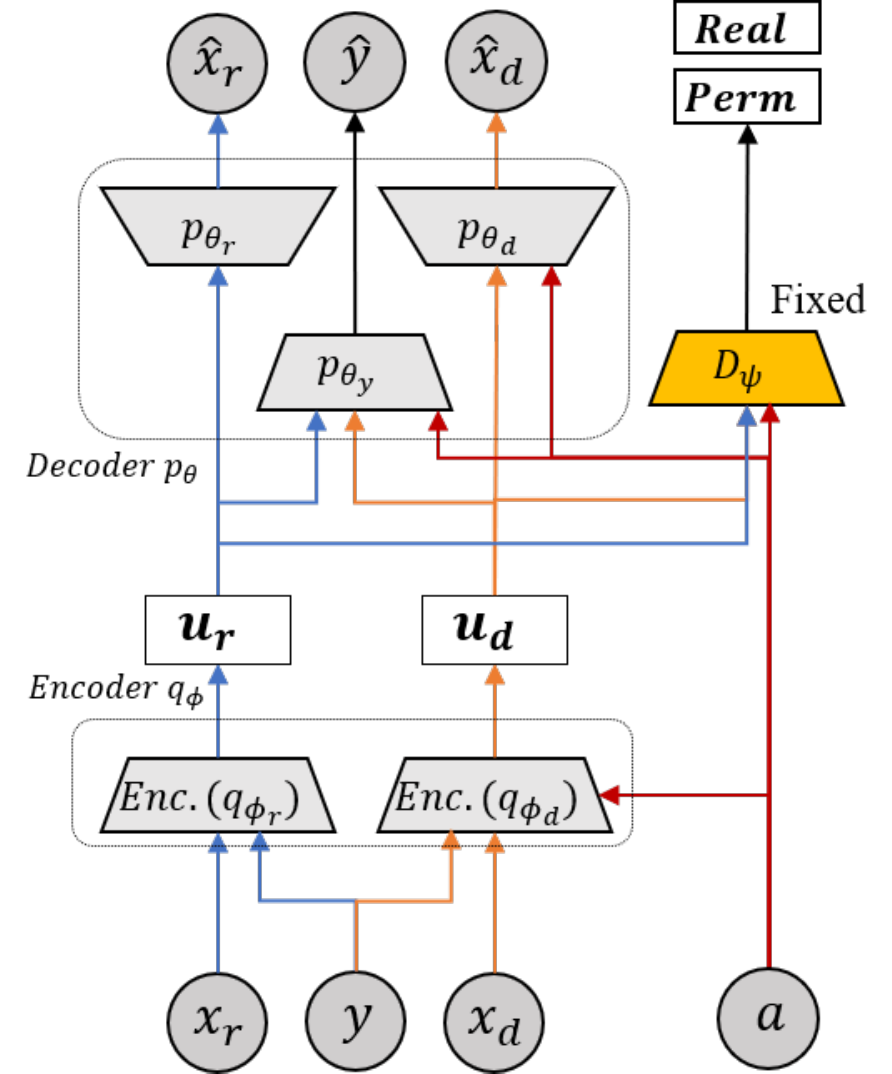}
    \caption{Neural Network (in min phase)}
	\label{fig:mine_train_min}
    \end{subfigure}
  }\cr
}
\caption{(a,c,e) Causal graph of CEVAE, mCEVAE, and DCEVAE. If a domain and its causal graph is given, the endogenous variables, $x$, of the given causal graph is mapped to either $x_r$ or $x_d$; (b,d,f) Bayesian network corresponding to the causal graph of CEVAE, mCEVAE, and DCEVAE.Arrows with solid lines denote generative process, and arrows with dashed lines denote
inference process; (g) Neural network structure of DCEVAE in max phase ($\max \mathcal{M}_{D}$) (h) Neural network structure of DCEVAE in min phase ($\min \mathcal{L}_{\textit{DCEVAE}}$)}
\end{figure*}
\subsection{Causality with Variational Autoencoder}
\citet{louizos2017causal} showed that modeling exogenous variable, $U$, in a causal graph can be interpreted as an inference task on the latent variables in variational autoencoder (VAE) \cite{kingma2013auto}. The evidence lower bound (ELBO), $\mathcal{M}_{ELBO}$, of VAE is derived as $\log{p(x)}\geq \mathbb{E}_{q_{\phi}(u|x)}\left[p_{\theta}(x|u)+p(u)-\log{q_{\phi}(u|x)}\right]=:\mathcal{M}_{ELBO}$

\citet{louizos2017causal} suggested the modified ELBO in Eq. \ref{eq:cevae_elbo}, based on a causal graph. In Causal Effect Variational Autoencoder (CEVAE), $a$ is correlated with all $x$, and $a$ does not deterministically cause the $x$. If a causal graph models the descendant of $a$ in $x$, the causality from $a$ to $x$ will be embedded in $u$ by $q_\phi(u|a,x,y)$ in ELBO. This embedded $a$ in $u$ interrupts the counterfactual generation of $p(y|\neg a,u)$ because the negation only affects $a$, not the embedded components in $u$.
\begin{align}\label{eq:cevae_elbo}
\mathcal{M}_{\textit{CEVAE}} &= \mathbb{E}_{q(u|a,x,y)}[\log{p(u)} + \log{p(a,x|u)} \nonumber\\
&+ \log{p(y|a,u)} -\log{q(u|a,x,y)}] \nonumber\\
& +{\log{q(a^*|x^*)}} +\log{q(y^*|x^*,a^*)}~,
\raisetag{1\baselineskip}
\end{align}
where $x$, $a^*$, $y^*$ being the observed values in the training set.

To compensate this potential problem in CEVAE, modified version of CEVAE \cite{pfohl2019counterfactual}, or mCEVAE, assumed that $x$ and $y$ are caused by $a$ and $u$. $\mathcal{M}_{\textit{mCEVAE}}$ uses the maximum mean discrepancy (MMD) to regularize the generations to remove the information of $a$ from $u$, but this MMD regularization removes $u$ components that is simply correlated to $a$, not caused by $a$. 
\begin{align}\label{eq:mcevae_elbo}
\mathcal{M}_{\textit{mCEVAE}} &=\mathbb{E}_{q(u|a,x)}\big[\lambda_{x}\log{p(x|a,u)}+\lambda_{y}\log{p(y|a,u)}\big]\nonumber\\
&+\lambda_{1} D_{MMD}\big(q_\phi(u)||p(u)\big)\nonumber\\
&+\lambda_{2} \sum_{a_k \in \mathcal{A}}D_{MMD}\big(q_\phi(u|a=a_k)||p(u)\big),
\raisetag{1.5\baselineskip}
\end{align}
We hypothesize that the accurate counterfactual generation lies in the middle of these two models. We define $x_d \subset x$ is a subset of features caused by $a$ whereas $x_r \subset x$ is the other subset of irrelevant features to the intervention. Similarly, we define the exogenous variables of $x_r$ and $x_d$ to be $u_r$ and $u_d$, respectively. When the counterfactual generation is required, the intervention on $a$ should be imposed on $x_d$, and $x_r$ should be maintained. This disentanglement is the fundamental motivation of our model, DCEVAE.

Also, CEVAE separates the decoder network into two alternative functions: $f_{a=0}$ and $f_{a=1}$. Therefore, either decoder tends to be updated according to the observed sensitive variables. For example, $i_{th}$ data $(x_i, a_i=1,y_i)$ is utilized to learn the parameters of ${f}_{a=1}$. When CEVAE makes the counterfactual samples, the latent values from a data instance with $a$ are propagated to the decoder $f_{\neg a}$. However, the lack of examples of $a$ in the training process of $f_{\neg a}$ can cause inaccuracy. To resolve this issue, mCEVAE regularizes latent variables $u_{a=0}$ and $u_{a=1}$ to be similar by using 
\subsection{Disentanglement with Total Correlation}
A latent variable, $u$, is considered to be disentangled if $u_j$ are independent when $j$ indicates a dimension of the variable. where the total correlation is ${TC} = KL\big(q(u)||\Pi_{j=1}^{d}q(u_j)\big)$, where $d$ is a dimension of $u$. When we minimize the TC, the latent is disentangled \cite{kim2018disentangling}.

\section{Methodology}

\subsection{Causal Structure of Disentangled Causal Effect Variational Autoencoder}
\noindent We design the structure of DCEVAE from the causal graph in Figure \ref{fig:mine_causal}. The causal graph specifies the separated causalities from the sensitive variable, $a$, to the feature variables, $x$. Hence, $x$ is divided into the feature variables caused by sensitive information, $x_d$; and the other feature variables, $x_r$. This paper assumes that the causal graph gives the attribute association to either $x_d$ or $x_r$ from the domain. For instance, in Figure \ref{fig:overall}, if we regard $Male$ as a sensitive variable because $Male$ indicates the gender, $x_d$ becomes the set of its descendant variables, $x_d=\{Mustache, Smiling, ...\}$ in the domain causal graph, and $x_r$ is set to be the complementary set of $x_d$ in causal graph variables.

This separation also introduces two corresponding exogenous variables: $u_d$ and $u_r$. As $u_d$ and $u_r$ are exogenous variables in Figure \ref{fig:mine_causal} where we assume that $u_r$ and $u_d$ are disentangled. Also, we assumed that $x_d$ are affected by $a$, not correlated with $a$; so $u_d$ needs to be disentangled with $a$ because $x_d$ will be deterministically caused by $a$ as an endogenous factor. On the other hand, $u_r$, which causes $x_r$, may hold the correlated information of $a$, so we did not disentangled it with $a$ since $x_r$ is correlated with $a$.

The usual set up in potential outcome framework is intervention $a$ precedes outcome $y$, and all features $x$ are precedes $a$. However, $x$ being preceded to $a$ is a strong assumption for the real-world case. The more general case is some of $x$ occur before $a$, the exogenous $u_r$ $\rightarrow$ $x_r$, and the rest of $x$ come after $a$, the exogenous $u_d$ and $a$ will have a common child $x_d$. For instance, let us assume that there is a woman ($a$: \textit{Gender}) who went to a women's only school ($x_d$: \textit{School}). When we compute counterfactual value by intervening $a=\textit{Male}$, then the value of school will change. However, this person's birth year ($x_r$), which is not a descendant of gender, will not be changed. This consideration enables a more general causal graph ordering assumption to be incorporated for counterfactual generation based on VAEs.

\subsection{Bayesian Network of Disentangled Causal Effect Variational Autoencoder}
The causal graph in Figure \ref{fig:mine_causal} translates to a Bayesian network in Figure \ref{fig:mine_bayes}. Two exogenous variables, $u_r$ and $u_d$, in the causal graph are translated into two latent variables in the Bayesian network. This Bayesian network corresponds to the neural network in Figure \ref{fig:mine_train_min}. The neural network consists of an inference structure on two latent random variables ($q_{\phi_r}$ and $q_{\phi_d}$), and the neural network also includes a structure for disentanglement, as $D_\psi$. The objective function of DCEVAE, Eq \ref{eq:whole_loss}, is devised to satisfy the above model structure. The latent variables are inferred by the optimization of $\mathcal{M}_{\textit{ELBO}}$. The disentanglement of the latent variables is resolved by reducing the total correlation, $\mathcal{L}_{TC}$. In the computation of the total correlation, we use the discriminator $D_\psi$ from $\mathcal{M}_{D}$, so we add an optimization on $\mathcal{M}_{D}$ to our objective function, as well. Eventually, our objective becomes the min-max structure to correspond to the counterfactual generation and the latent disentanglement.
\begin{equation}
\begin{aligned}
\ & min_{\theta, \phi}{\ \mathcal{L}_{\textit{DCEVAE}}} := -\mathcal{M}_{\textit{ELBO}} + \beta_{tc} \mathcal{L}_{TC}	\\
\ & max_{\psi}{\ \mathcal{M}_{D}}
\label{eq:whole_loss}
\end{aligned}
\end{equation}
\subsection{Counterfactual Inference and Counterfactual Example Generation}
\noindent Besides the model structure, a causal graph requires the inference on the exogenous variable to estimate the causal effect, including the total effect and the counterfactual effect. We match two types of exogenous variables, $u_d$ and $u_r$, in the causal graph to two corresponding latent variables of DCEVAE. According to the counterfactual inference \cite{pearl2009causality}, a counterfactual $y$ can be inferred in following steps:
\begin{enumerate}
\itemsep0em
  \item \textbf{Abduction} Infer the distribution of $u_d$ and $u_r$ with the encoder network of DCEVAE: $q(u_d|a, x_d, y)$ and $q(u_r|a, x_r,y)$.
  \item \textbf{Action} Substitue $A$ with $\neg a$.
  \item \textbf{Prediction} Compute the probability of counterfactual $y$, with the decoder network of DCEVAE, $p(y|\neg a, u_d, u_r)$.
\end{enumerate}
In the third step, prediction, $x_d$ could be potentially change for intervention $\neg a$, but $x_r$ maintain its value because we consider causal ordering as $u_r\rightarrow x_r$ and $\left(a, u_d\right) \rightarrow x_d$.\\
The identification is shown as follows.
\begin{proposition}
If we recover $p(u_d,u_r,a,x_d,x_r,y)$, then we recover counterfactual effect under the causal model Fig \ref{fig:mine_bayes}.
\end{proposition}
\begin{proof} $p(y\mid x_d,x_r,do(a=1))$
\begin{equation}
\begin{aligned}
&=\int_{u_d}\int_{u_r}p\big(y\mid x_d,x_r,do(a=1),u_d,u_r\big) \nonumber\\
&~~~~~~~~~~~~~~~~\times p\big(u_d,u_r\mid x_d,x_r,do(a=1)\big)du_r du_d \\
&=\int_{u_d}\int_{u_r}p(y\mid x_d,x_r,a=1,u_d,u_r) \nonumber\\
&~~~~~~~~~~~~~~~~\times p(u_d,u_r\mid x_d,x_r,a=1)du_r du_d.\nonumber
\end{aligned}
\end{equation}
$p(y| x_d,x_r,a=1,u_d,u_r)$ and $p(u_d,u_r|x_d,x_r,a=1)$ can be identified from the distribution of $p(u_d,u_r, a,x_d,x_r,y)$.
\end{proof}

\subsection{Evidence Lower Bound of DCEVAE}
\noindent We propose the ELBO to disentangle $u_d$ and $u_r$ by following the Bayesian network structure, Figure \ref{fig:mine_bayes}. We assume that $x_d$ and $u_r$ is independent given $u_d$; and $x_r$ and ($a$, $u_d$) are independent given $u_r$, as well. Then, the decoder distribution, $p_{\theta}(x_d, x_r, y, u_d, u_r|a)$, can be factorized as the below:
\begin{align}\label{eq:prior}
p_{\theta}(x_d, x_r, y, u_d, u_r|a)&=p(u_d)p(u_r)p_{\theta}(x_d|a,u_d)\nonumber\\
& ~~~~ \times p_{\theta}(x_r|u_r) p_{\theta}(y|a,u_d,u_r)
\end{align}
Also, we assume that the posterior, $q_\phi(u_d,u_r|a,x_d,x_r,y)$, can be factorized as the Eq. \ref{eq:posterior}.
\begin{equation}
q_\phi(u_d,u_r|a,x_d,x_r,y)=q_\phi(u_d|a,x_d,y)q_\phi(u_r|a,x_r,y)
\label{eq:posterior}
\end{equation}
Given an approximate posterior $q_\phi(u_d,u_r|a,x_d,x_r,y)$, we obtain the variational lower bound as Eq. \ref{eq:ELBO_jensen}.
\begin{align}
    & \log p_{\theta}(x_d, x_r, y|a) \nonumber\\
    & \geq \mathbb{E}_{q_\phi(u_d|a,x_d,y)q_\phi(u_r|a,x_r,y)}\big[\log{p_\theta\left(y|a,u_d,u_r\right)}\big] \nonumber\\
    & +\mathbb{E}_{q_\phi(u_d|a,x_d,y)}\big[\log{p_\theta(x_d|a,u_d)}\big] \nonumber\\
&+ \mathbb{E}_{q_\phi(u_r|a,x_r,y)}\big[\log{p_\theta(x_r|u_r)}\big] \nonumber\\
    &+KL\big(q_\phi(u_d|a,x_d,y)||p(u_d)\big)+KL\big(q_\phi(u_r|a,x_r,y)||p(u_r)\big) \nonumber\\
	&=: \mathcal{M}_{\textit{ELBO}}
\label{eq:ELBO_jensen}
\end{align}
In practice, we use the neural network layers to infer the parameters of a Gaussian distribution over the joint space of $u_d$ and $u_r$. To obtain the posterior distribution of $q_\phi$, $p(u_d)$, and $p(u_r)$ are the prior distributions following the Gaussian distribution, and we utilize the reparametrization, accordingly. The below is the encoder structure for $u_d$ and $u_r$.
\begin{align}
p(u_d)=\mathcal{N}(u_d|0,I);& \quad p(u_r)=\mathcal{N}(u_r|0,I); \nonumber\\ 
q_\phi(u_d|a,x_d,y)=\mathcal{N}(\bar{\mu}_d,\bar{\sigma}_d^2I);& \quad q_\phi(u_r|a, x_r,y)=\mathcal{N}(\bar{\mu}_r,\bar{\sigma}_r^2I); \nonumber\\
\bar{\mu}_d=g^{\mu}_d(a,x_d,y);& \quad \bar{\sigma}_d=g^{\sigma}_d(a,x_d,y); \\ 
\bar{\mu}_r=g^{\mu}_r(a,x_r,y);& \quad \bar{\sigma}_r=g^{\sigma}_r(a,x_r,y), \nonumber
\end{align}

With decoder $p_\theta$, DCEVAE provides the generative process for $a$, $x_d$, $x_r$, and $y$. Since $a$, $x$, and $y$ are different in nature, we differentiate their distributions. We assume that if $x$ are continuous variables, they follow the Gaussian distribution. We let $y$ be the binary variable of the Bernoulli trial. $x_d$ and $y$ are influenced by $a$. On the other hand, $x_r$ is only determined by $u_r$ and not determined by $a$. The decoder structure for a factual data instance is specified as below:
\begin{align}
p_{\theta_d}(x_d|a,u_d,y)&=\mathcal{N}(\mu=\hat{\mu}_d,{\sigma}^2=\hat{\sigma}_d^2 I); \nonumber\\
p_{\theta_r}(x_r|u_r)&=\mathcal{N}(\mu=\hat{\mu}_r,{\sigma}^2=\hat{\sigma}_r^2 I);\\
\hat{\mu}_r, \hat{\sigma}_r^2=&f_r^{\mu}(u_r), f_r^{\sigma}(u_r)\nonumber\\
\hat{\mu}_d, \hat{\sigma}_d^2=&f_d^{\mu}(a,u_d), f_d^{\sigma}(a,u_d)\quad \nonumber\\
p_{\theta_y}(y|a,u_d,u_r)=Bern&(\pi=\hat{\pi}_y); \quad \hat{\pi}_y=f_{y}(a,u_d,u_r)\nonumber
\end{align}

For counterfactual data generation, we use the same encoder and decoder structure. The only difference is using $\neg a$ for a decoder output. As mentioned before, a counterfactual $x_r$ is same with the factual $x_r$, while a counterfactual $x_d$ and $y$ are influenced by the counterfactual $\neg a$.
\subsection{Disentanglement Loss of DCEVAE}
\label{sec: tc}
Up to this point, we treat each pair ($a$, $u_d$) and ($u_d$, $u_r$) to be disentangled. This assumption minimize $\mathcal{L}_{\textit{TC}} = KL\big(q(a, u_r, u_d)||q(a, u_r)q(u_d)\big)$. This KL divergence is intractable because both $q(a,u_d,u_r)$ and $q(a, u_r)q(u_d)$ are conditioned on $a$, $x_d$, and $x_r$. Therefore, we take an alternative approach adopted in FactorVAE \cite{kim2018disentangling}. Algorithm 1 in Appendix 3.3 specifies the sampling from $q(a,u_r,u_d)$ and $q(a, u_r)q(u_d)$ under the conditions, and we apply the permutation to minimize $\mathcal{L}_{\textit{TC}}$. $\mathcal{L}_{\textit{TC}}$ is related to the discriminator $D_{\psi}$ by the \textit{density-ratio trick}, which is approximated by a neural network. The output of $D_{\psi}([a,u_d,u_r])$ estimates the probability when the density function takes a sampled input from $q(a,u_d,u_r)$, rather than from $q(a,u_r)q(u_d)$. $\mathcal{L}_{\textit{TC}}$ is expressed by $D_{\psi}$ as in Eq. \ref{eq:TC_loss}.
\begin{align}\label{eq:TC_loss}
\mathcal{L}_{\textit{TC}}&=KL\big(q(a,u_d,u_r)\| q(a, u_r)q(u_d)\big)\nonumber\\
&\approx \mathbb{E}_{q(a,u_d,u_r)}\left[\log \frac{D_{\psi}(a,u_d,u_r)}{1-D_{\psi}(a,u_d,u_r)}\right].
\end{align}
For training the network $D_\psi$, we should maximize $\mathcal{M}_{D_{\psi}}$.
\begin{align}
\mathcal{M}_{D_{\psi}} &= \mathbb{E}_{q(a,u_d,u_r)}\big[\log(D_\psi([a,u_d, u_r]))\big]\nonumber\\
&+\mathbb{E}_{q(a,u_r)q(u_d)}\big[\log{1-D_\psi([a,u_d,u_r])}\big].
\end{align}
Appendix 3.3 provides a whole algorithm of DCEVAE including the minimization of $\mathcal{L}_{\textit{DCEVAE}}$ and the maximization phase of $\mathcal{M}_D$.

\subsection{Theoretic Analysis on Covariance Structure}
Eq.\ref{eq:pPCA} defines the posterior distribution of the latent variable, and we derive the stationary point of $\Sigma$. \cite{lucas2019don}.
\begin{equation}
\begin{aligned}
\label{eq:pPCA}
q(u|a,x,y) &= \mathcal{N}\big( V_u\left([a,x,y]-\mu \right), \Sigma \big) \\
\end{aligned}
\end{equation}
Here, $u$ is $[u_r, u_d]$, and $\Sigma$ is a covariance matrix of the joint distribution of latent variables. $\bar{\Sigma}$ is a covariance matrix of the permuted $u$ by discriminator $D_{\psi}$. 

If $[u_r, u_d]$ is well disentangled, $\Sigma$ should show two block diagonal matrices corresponding to $u_r$ and $u_d$. To theoretically analyze the disentangling, we provide Proposition \ref{proposition1} and Eq. \ref{eq: DCEVAE_cov}.
\begin{proposition}
\label{proposition1}
Let $\Sigma^*$ be the stationary point of the $\mathcal{L}_{\text{DCEVAE}}$. For a linear DCEVAE, $\Sigma^*$ has the form:
\begin{align}\label{eq: DCEVAE_cov}
&\Sigma^*=\Big\{\frac{1}{1+\beta}\big(\frac{1}{\sigma^2}W_r^TM_r^TW_r+\frac{1}{\sigma^2}W_d^TM_d^TW_d\\
&~~~~~+diag(\frac{1}{\sigma^2}W_y^TW_y)+I+\beta ({{\bar{\Sigma}}^{-1}})^{T}\big)\Big\}^{-1}\text{, where} \nonumber\\
& M_r = \begin{pmatrix} \small 
I_{n \times n} & 0_{n \times m} \nonumber\\ 
0_{m \times n} & 0_{m \times m} \nonumber
\end{pmatrix}, \nonumber
M_d = \begin{pmatrix} \small 
0_{n \times n} & 0_{n \times m} \nonumber\\ 
0_{m \times n} & I_{m \times m} \nonumber
\end{pmatrix}
\end{align}with $n=|u_r|$ and $m=|u_d|$ .
\end{proposition}
\begin{proof}
see Appendix 1.1 for the proof.
\end{proof}
We observed that $u_r$ and $u_d$ has their distinct covariance blocks from the real dataset due to the masking effect of $M_r$ and $M_d$. This theoretically shows the disentangling effects of DCEVAE. Despite the clear disentangling effect from the masks, the off-diagonal covariance can be feasible by $\bar{\Sigma}$, so we provide Corollary \ref{corollary1}. Also, the part of covariance, ${M_d}^T\bar{\Sigma}$, is constructed to be independent to $a$ by the Total Correlation loss, $\mathcal{L}_{\textit{TC}}$, which enforces $u_d$ to be independent to $a$.
\begin{corollary}
\label{corollary1}
As $\beta \rightarrow \infty$, $\Sigma^*$ becomes a covariance matrix with two blocks on diagonal.
\end{corollary}
\begin{proof}
$\beta \rightarrow \infty \Rightarrow \frac{1}{1+\beta} \rightarrow 0, \frac{\beta}{1+\beta} \rightarrow 1 \Rightarrow \Sigma^* \rightarrow \bar{\Sigma}$.
Also, $\bar{\Sigma}$ is designed to permute $u_d$ and $u_r$ by following Line 8-11, Algorithm 1, Appendix 3.3; $cor(u_d,u_r) \rightarrow 0$. Therefore, $\Sigma^*$ has two blocks of $u_d$ and $u_r$ dimensions.
\end{proof}
Figure \ref{fig:cov0_cevae} and \ref{fig:cov1_cevae} contrast the covariance structure of DCEVAE to the CEVAE. There is no disentanglement effect on the covariance of CEVAE, so it cannot distinguish the correlated latent variable from the caused ones. Theoretic analyses on the covariance of CEVAE are in Appendix 1.2.
\begin{figure}[ht]
\begin{subfigure}[b]{.115\textwidth}
  \centering
  \includegraphics[width=0.98\linewidth]{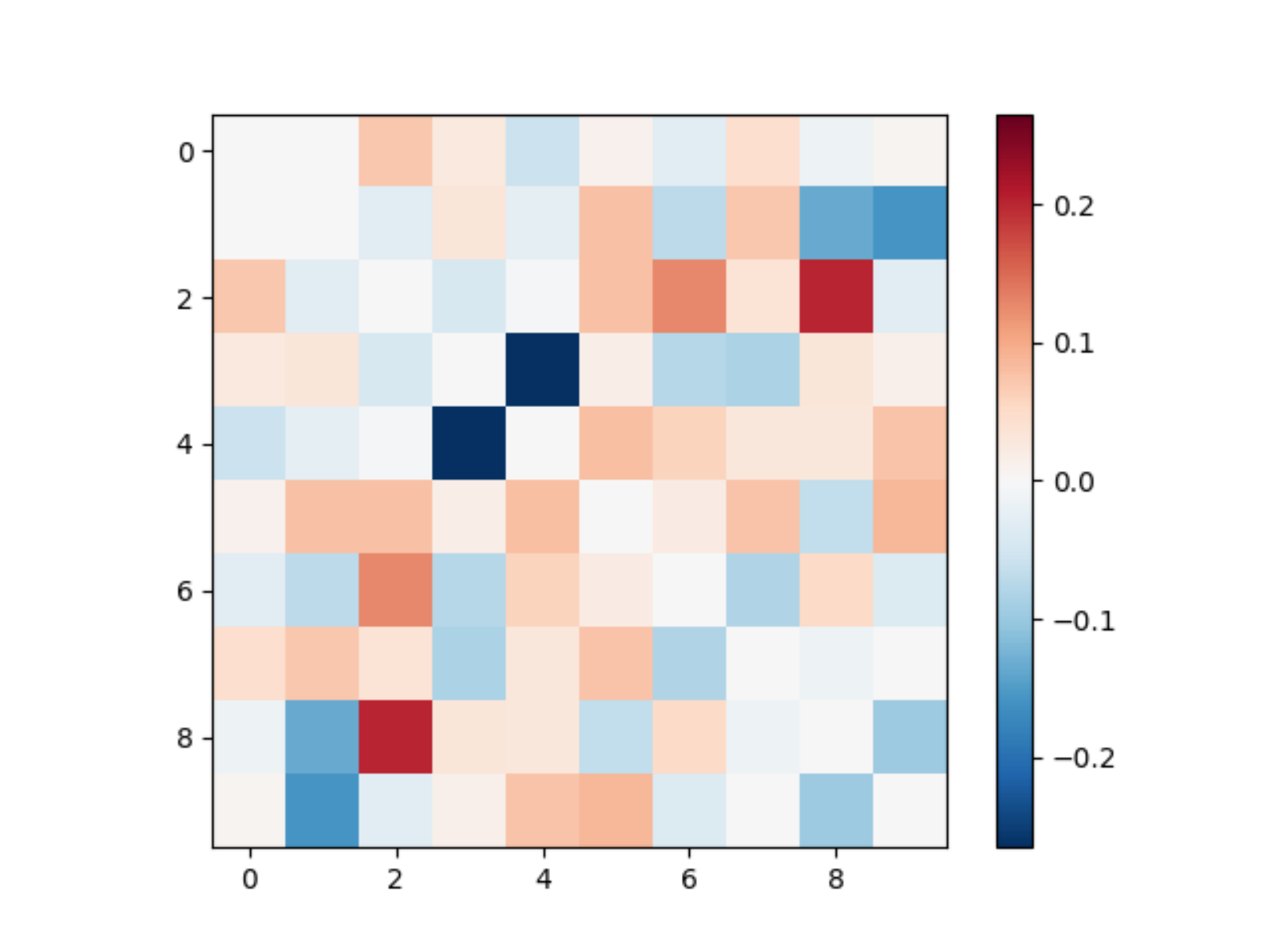}
  \caption{$\beta=0$}
\end{subfigure}
\begin{subfigure}[b]{.115\textwidth}
  \centering
  \includegraphics[width=0.98\linewidth]{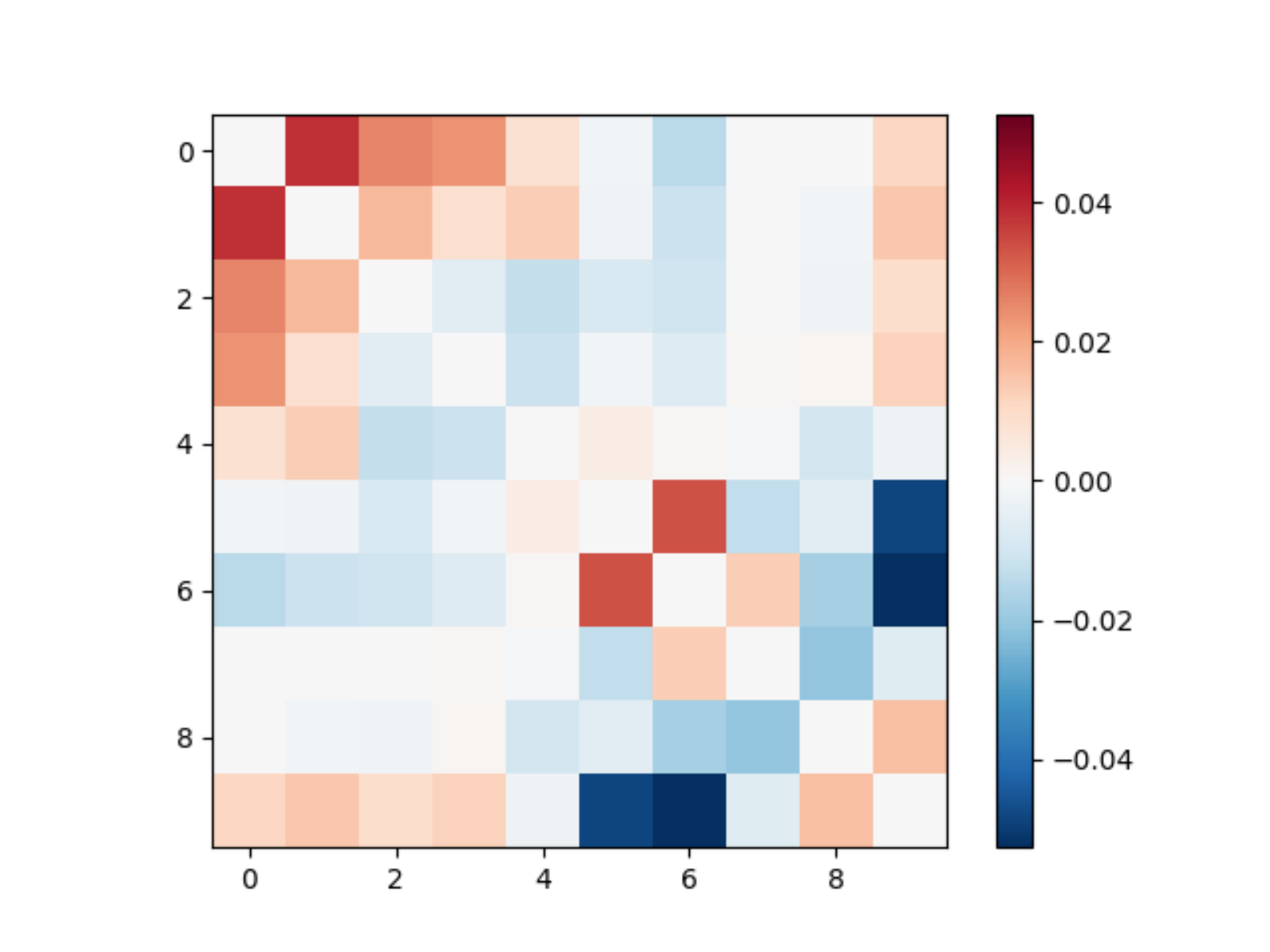}
  \caption{$\beta=1000$}
\end{subfigure}
\begin{subfigure}[b]{.115\textwidth}
  \centering
  \includegraphics[width=0.98\linewidth]{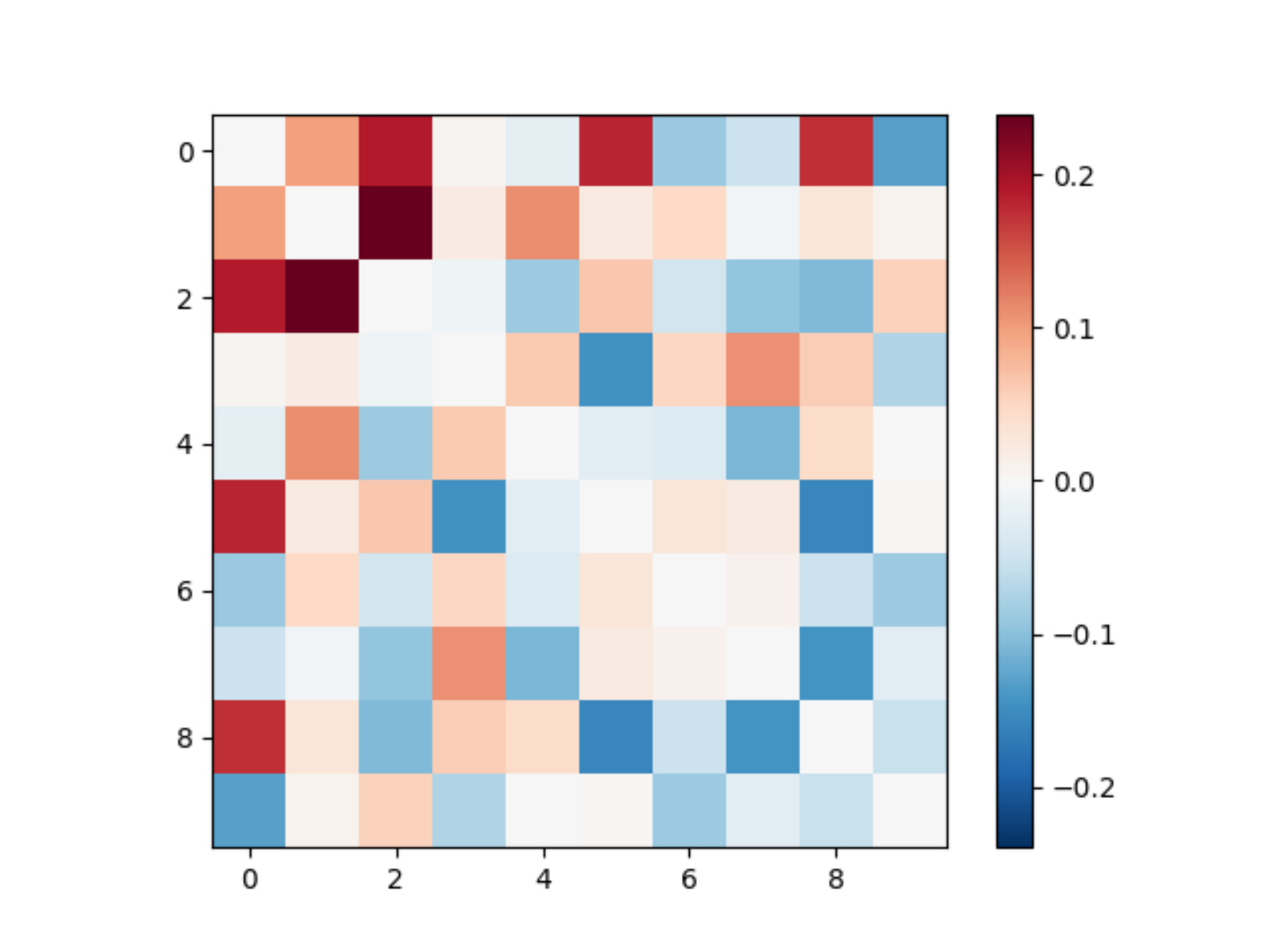}
  \caption{$\Sigma_0$}
\label{fig:cov0_cevae}
\end{subfigure}
\begin{subfigure}[b]{.115\textwidth}
  \centering
  \includegraphics[width=0.98\linewidth]{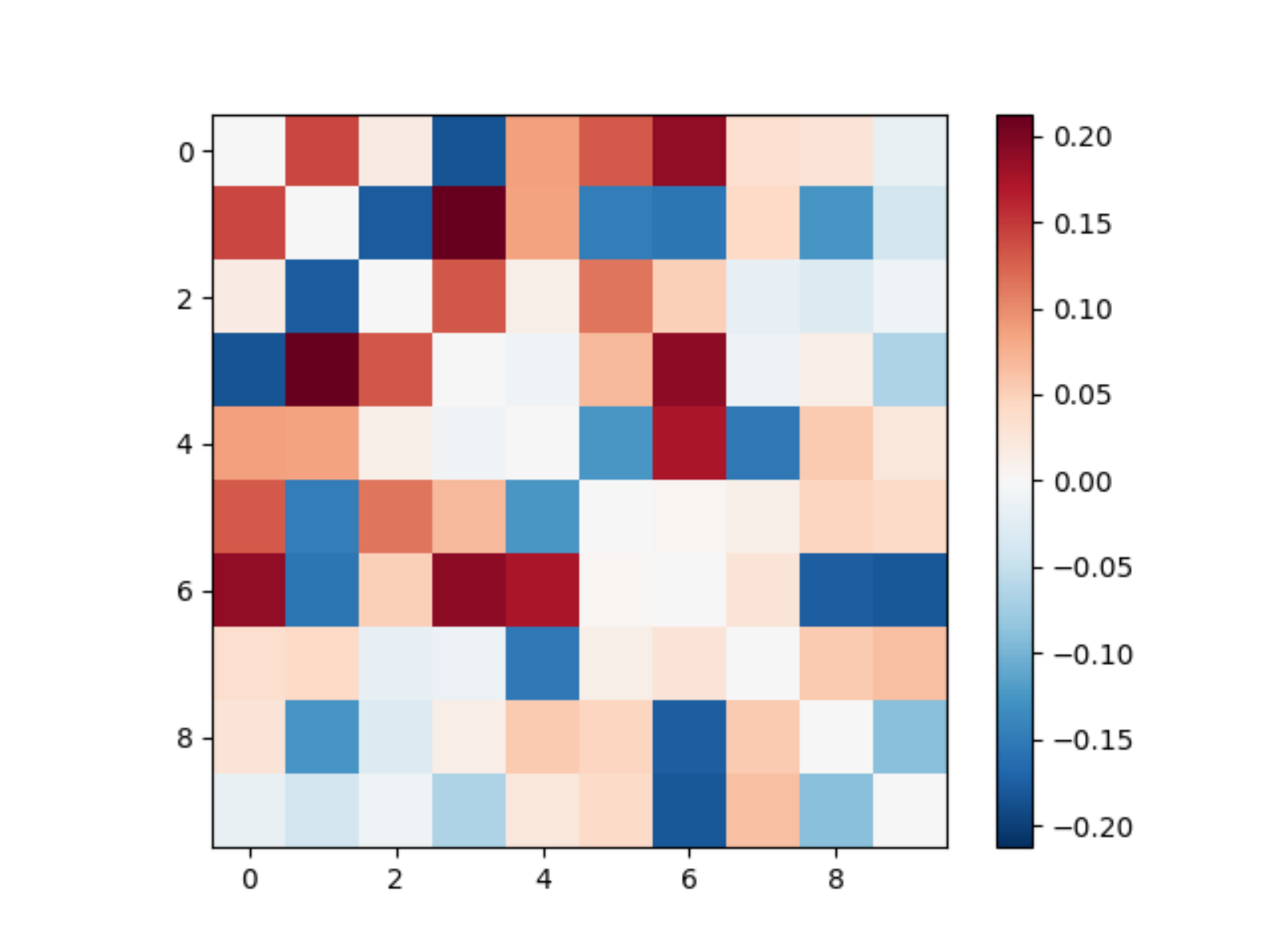}
  \caption{$\Sigma_1$}
\label{fig:cov1_cevae}
\end{subfigure}
\caption{(a,b) The covariance matrix of sampled latent values from DCEVAE with $\beta=0$ and $1000$. The dimension of $u_d$ is five, and so is $u_r$. (c,d) Each covariance matrix $\Sigma_0$ and $\Sigma_1$ of sampled latent values from each decoder of CEVAE.}
\end{figure}

\subsection{Application}
\begin{table*}[ht]{\small
\renewcommand{\tabcolsep}{0.1mm}
  \begin{tabular}{c|c|cccc|c|c|cc}
	\toprule
          &\multirow{2}{*}{Total Effect}
          &\multicolumn{4}{c|}{Counterfactual Effect (CE)}  &\multirow{2}{*}{CE error} &\multirow{2}{*}{${\chi^2}$}
          &\multicolumn{2}{c}{Accuracy} \\
        \cmidrule{3-6}
	\cmidrule{9-10}
              &     &$o_{00}$   &$o_{01}$   &$o_{10}$	&$o_{11}$	&		&{\tiny $(\times 10^{-2})$}	&LR	&SVM\\
    \midrule
    Real Data   &.1936 &.1785 &.1266 &.1293 &.2023	&-	&0  &.8158 &.8109\\
    \midrule
    CausalGAN 	       &$.1959_{\pm.0140}$	         &$.1123_{\pm.0231}$		&$.1437_{\pm.0128}$	&$.1690_{\pm.0144}$	&$.2080_{\pm.0163}$	&$.1374_{\pm.0267}$	&$4.20_{\pm.08}$	&$.7805_{\pm.0220}$	&$.8126_{\pm.0026}$\\
    \midrule
	\midrule
    CausalGAN-IC    &$.2432_{\pm.0163}$				&$\underline{.1585_{\pm.0622}}$		&$.1912_{\pm.0347}$	&$.2267_{\pm.0373}$	&$.2556_{\pm.0162}$	&$.2689_{\pm.0663}$	&$4.65_{\pm.21}$	&$\underline{.7994_{\pm.0141}}$	&$\mathbf{.8049_{\pm.0129}}$\\	
    CVAE   		       &$\underline{.1510_{\pm.0279}}$	     		    &$.1228_{\pm.0177}$		&$\mathbf{.1194_{\pm.0241}}$	&$\underline{.1347_{\pm.0264}}$	&$.1452_{\pm.0295}$	&$\underline{.1524_{\pm.0618}}$	&$\underline{0.25_{\pm.01}}$	&$.7945_{\pm.0040}$	&$.7886_{\pm.0064}$\\
    CEVAE   	       &$.0954_{\pm.0036}$			&$.0951_{\pm.0066}$		&$.0843_{\pm.0017}$	&$.0921_{\pm.0044}$	&$.0972_{\pm.0041}$	&$.3399_{\pm.1664}$	&$\mathbf{0.18_{\pm.01}}$	&$.7909_{\pm.0014}$	&$.7615_{\pm.0000}$	\\
    mCEVAE  	       &$.1480_{\pm.0160}$ 				&$.1296_{\pm.0858}$		&$\underline{.1166_{\pm.0422}}$	&$\mathbf{.1340_{\pm.0431}}$	&$\underline{.1543_{\pm.0164}}$	&$.1790_{\pm.1190}$	&$0.26_{\pm.03}$	&$.7877_{\pm.0199}$	&$.7884_{\pm.0246}$\\
    DCEVAE (ours)    &$\mathbf{.1831_{\pm.0011}}$ 		&$\mathbf{.1871_{\pm.0086}}$		&$.1574_{\pm.0054}$	&$.1673_{\pm.0068}$	&$\mathbf{.1878_{\pm.0010}}$	&$\mathbf{.0923_{\pm.0071}}$	&$0.27_{\pm.01}$	&$\mathbf{.8095_{\pm.0050}}$	&$\underline{.7995_{\pm.0003}}$	\\
	\bottomrule
  \end{tabular}}
	\caption{The total effect and counterfactual effect of real and generated datasets ($O=\{\text{race, native country}\}$). CE error is $\sum_{i,j\in\{0,1\}}|\frac{o_{ij}-o^*_{ij}}{4}|$ with true CE, $o^*$. The numbers in bold indicates the best performance, and the underlined numbers indicate the second best performance.}
	\label{table:causal}
\end{table*}

\noindent \textbf{Causal Fair Classification}
\label{sec: fair}
The task of causal fairness requires estimating $y$ with the minimized influence of $a$.
When we assume $p(\hat{y}|a, u_{d,i}, u_{r,i})-p(\hat{y}|\neg a, u_{d,i}, u_{r,i})=0$ for $u_{d,i}$ and $u_{r,i}$ from $i$-th data instance; we say that the counterfactual fairness is satisfied for the data instance. Therefore, we alter the objective function of DCEVAE by adding a regularization $\mathcal{L}_{f}$ as the below:
\begin{align}\label{eq:fair_loss}
&\min{\mathcal{L}_{\textit{fair}}} = \mathcal{L}_{\textit{DCEVAE}} + \beta_f \mathcal{L}_{f}, ~
\max{\mathcal{M}_{D}} \hspace{.3em} \text{where} \\
&\mathcal{L}_{f}=\mathbb{E}_{q(u_d, u_r|a,x_d,x_r)}\big[||{p_\theta(y|a, u_d, u_r)-p_\theta(y|\neg{a}, u_d, u_r)}||_2\big] \nonumber
\end{align}
After optimizing Eq. \ref{eq:fair_loss}, we train a classifier, such as a logistic regression (LR), with the pairs of $\hat{y}$ with $\hat{x}$ through the decoder, $p_\theta(x_d, x_r, y|a, u_d, u_r)$ and $p_\theta(x_d, x_r, y|\neg a, u_d, u_r)$. The test procedure utilizes the raw input of the testing feature $a$ and $x$ without any modifications.

\noindent \textbf{Counterfactual Image Generation} The counterfactual image generation task allocates $y$ to be the image and $x$ to be the labels which describe the image. $a$ is the label which we want to intervene on. Unlike the fairness dataset, we modified the encoder, $q_\phi(u_d,u_r|a,y)$, because the information of $x$ is already embedded on an image, $y$. The counterfactual images are sampled from the decoder, $p_\theta(y|\neg a, u_d, u_r)$, while $u_d$ and $u_r$ are obtained from the encoder.
\section{Experiments}
\subsection{Datasets and Baselines}
Appendix 4.1 provides the details of datasets; and Appendix 4.2 enumerates the causal graphs and their paired attributes, $x_d$ and $x_r$.\\
\noindent \textbf{Causal Estimation and Fair Classification} We use the UCI Adult income \cite{asuncion2007uci} dataset for causal estimation and fair classification tasks. We treat \textit{gender} as a sensitive varible (or intervention) $a$; \textit{income} as the outcome $y$; \textit{race}, \textit{age}, and \textit{native country} as  $x_r$; and otehr variables as $x_d$.
We benchmark the causal effect estimation on five baselines: CausalGAN \cite{kocaoglu2017causalgan}, CausalGAN-Incomplete (CausalGAN-IC), conditional VAE (CVAE) \cite{sohn2015learning}, CEVAE, and mCEVAE . CausalGAN-IC has the same generator and discriminator structures as CausalGAN, but CausalGAN-IC used the same causal graph DCEVAE used. For fairness experiments, we additionally use Unawareness (CF-UA), Additive Noise (CF-AN) \cite{kusner2017counterfactual}, and CFGAN \cite{xu2019achieving} as baselines.

\noindent\textbf{Counterfactual Image Generation} We use the CelebA dataset \cite{liu2018large} for the counterfactual image generation. We treat \textit{Mustache} as an intervention attribute $a$; an image as $y$; and the other attributes as $x$. Appendix 5.1 provide a similar experiment with \textit{Smiling} as an intervention attribute. We seperate $x_d$ and $x_r$ as the assumed causal graph in Figure \ref{fig:overall}. We consider the following baselines for counterfactual image generations: CVAE , CEVAE, mCEVAE, Conditional GAN (CGAN) \cite{mirza2014conditional} with Wasserstein distance \cite{arjovsky2017wasserstein} (denoted as CWGAN), CausalGAN, and CausalGAN-IC.

\subsection{Evaluation Metrics}
\noindent We evaluate the performance on the causal estimation and fairness task with following metrics used in CFGAN: \textbf{Total effect}, $TE(a_1, a_0)=P(y_{a_1})-P(y_{a_0})$, measures the change likelihood of $A$ from $a_0$ to $a_1$ on $Y$. \textbf{Counterfactual effect}, $CE(a_1, a_0|o)=P(y_{a_1}|o)-P(y_{a_0}|o)$, is the total effect conditioned on the observation, $o$. A \textbf{Logistic Regression (LR)} and a \textbf{Support Vector Machine (SVM)} are trained with generated datasets from the model, and their test accuracy is also used from the original dataset. \textbf{Chi square distance ($\chi^2$)} indicates the similarity between the generated and the real datasets \cite{daliri2013chi}.

We follows MaskGAN \cite{lee2019maskgan} to evaluate quality of counterfactually generated images:
\textbf{Semantic-level Evaluation} evaluates generated images by measuring the preservation of the original values of $x_r$. We trained a classifier with ResNet-18 \cite{he2016deep} to examine the accuracy of $x_r$ in generated images. \textbf{Distribution-level Evaluation} measures the quality and the diversity of generated images, and we used the Frechet Inception Distance (FID) \cite{cao2013facewarehouse}. \textbf{Identity Preserving Evaluation} evaluates the identity preservation ability, so we conducted a face matching experiment for whole pairs of original and counterfactual images with ArcFace \cite{deng2019arcface}. We use $1,544$ pairs for \textit{Mustache} and $10,000$ pairs for \textit{Smiling}.
\section{Results}
\subsection{Causal Estimation and Fairness Task}
Without the regularization of the fairness $\mathcal{L}_{\textit{fair}}$, we calculate the total effect (TE) and the counterfactual effect (CE) of our model and baselines in Table \ref{table:causal} for UCI Adult. DCEVAE estimates the total effect and the counterfactual effect close to the original dataset through DCEVAE does not know the exact causal graph structure, unlike CausalGAN. CausalGAN estimates the true TE and CE well, but CausalGAN-IC has lower CE and TE estimation accuracies when the incomplete causal graph is given.
CEVAE has lower performance on TE and CE estimations compared to DCEVAE, which is caused by the latent variable in CEVAE with the correlated information from $a$ to $X_d$. Therefore, CEVAE maintains correlated information even in the counterfactual prediction.

\begin{figure}[ht]
  \centering
  \includegraphics[width=.80\linewidth]{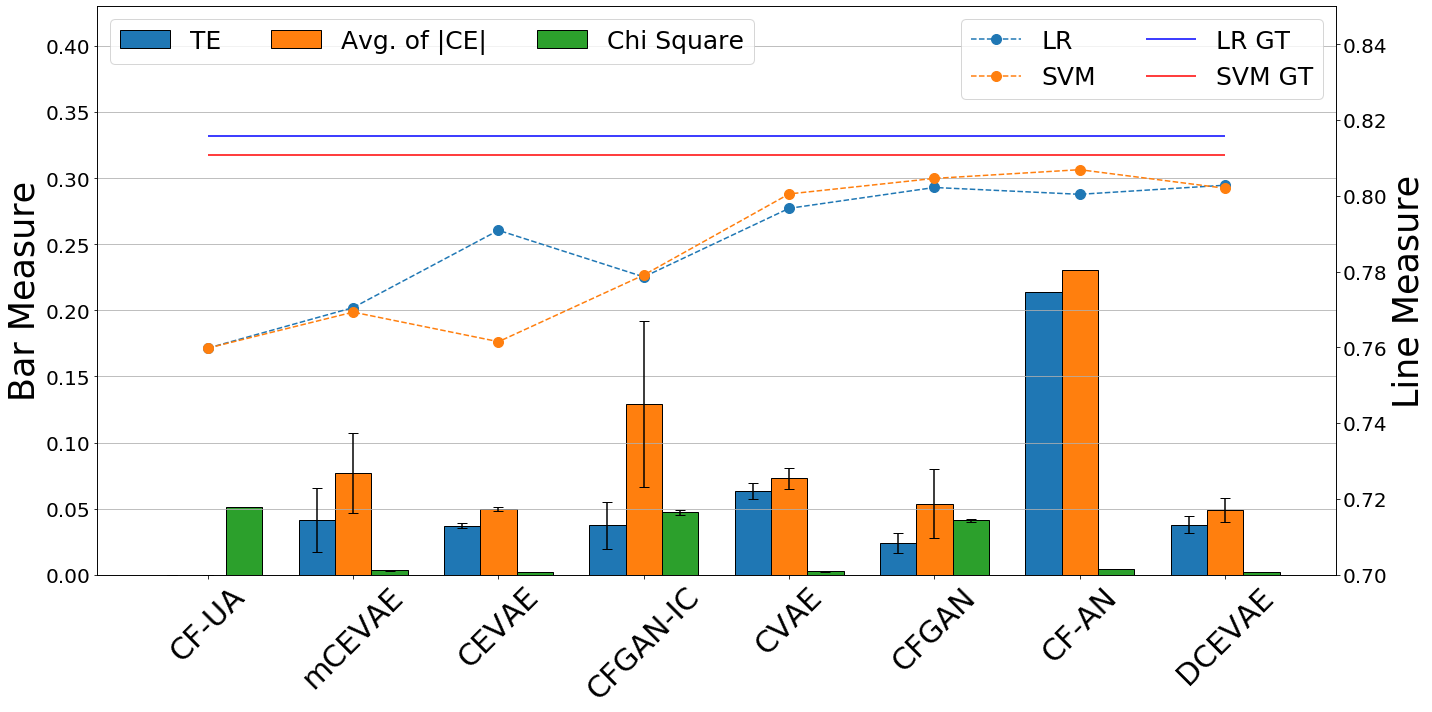}  
  \caption{(Left axis) TE, Average of the absolute value of CE, and $\chi^2$ (Chi-Square), (Right axis) Accuracy of LR and SVM For fairly generated dataset. GT is ground truth.}
  \label{fig:fair_bar_chart}
\end{figure}
With the regularization of the fairness, $\mathcal{L}_{\textit{fair}}$, Fig \ref{fig:fair_bar_chart} shows the results of fairness tasks for UCI Adult. The training datasets from CFGAN have the highest accuracy when LR and SVM are selected as classifiers, and CFGAN has a low value of TE and the average of CE. However, when an incomplete graph is given, CFGAN-IC has the low accuracy of LR and SVM, and the high values of TE and CE. In contrast to the low reliability of CFGAN, DCEVAE has comparable accuracies, TE and CE to CFGAN without causal graph structure. Also, the variance of TE and CE from DCEVAE shows the reliability compared to DCEVAE.

\begin{table*}[ht]
\renewcommand{\tabcolsep}{0.7mm}
\begin{center}{\small
\adjustbox{width=0.99\textwidth}{
  \begin{tabular}{c|c|c|cccccccc|c|c}
    \toprule
   \multirow{2}{*}{}	&\multirow{2}{*}{Model}   &\multirow{2}{*}{Target}	&\multicolumn{8}{c|}{Attribute clssification accuracy (\%)}  &\multirow{2}{*}{IP} &\multirow{2}{*}{FID}	\\
	\cmidrule{4-11}
	&	&	&WL	&MSO	&S	&ML	&E	&B	&NE	&Y	&	&	\\
	\midrule
	\multirow{4}{*}{Real}	
&CVAE		&$46.14_{\pm3.89}$	                  &$85.36_{\pm0.35}$	                  	&$73.72_{\pm0.70}$	&$\underline{80.51_{\pm1.18}}$	&$76.49_{\pm2.61}$	&$65.47_{\pm2.87}$	&$78.58_{\pm3.34}$	&$62.27_{\pm2.92}$	&$\underline{65.16_{\pm2.82}}$	&$0.26_{\pm0.04}$	&$187.88_{\pm6.09}$\\
&CEVAE		&$\mathbf{50.44_{\pm0.52}}$	&$85.27_{\pm0.22}$	                  	&$\underline{74.61_{\pm0.20}}$	&$79.86_{\pm0.51}$	&$\underline{77.01_{\pm2.14}}$	&$\underline{69.03_{\pm4.32}}$	&$\underline{82.05_{\pm1.10}}$	&$\underline{65.14_{\pm0.70}}$	&$64.48_{\pm2.04}$	&$\underline{0.29_{\pm0.04}}$&$181.29_{\pm8.57}$\\
&mCEVAE	&$44.26_{\pm4.51}$	                  &$\underline{85.38_{\pm0.27}}$	                  	&$73.41_{\pm1.33}$	&$79.81_{\pm1.11}$	&$76.63_{\pm1.54}$	&$59.88_{\pm1.83}$	&$79.16_{\pm0.80}$	&$63.30_{\pm4.24}$	&$65.04_{\pm1.87}$	&$0.23_{\pm0.02}$&$\mathbf{175.30_{\pm1.26}}$\\
&DCEVAE	&$\underline{49.68_{\pm0.18}}$	                  &$\mathbf{85.60_{\pm0.58}}$	&$\mathbf{74.75_{\pm0.42}}$	&$\mathbf{81.52_{\pm0.23}}$	&$\mathbf{78.89_{\pm0.53}}$	&$\mathbf{71.32_{\pm2.27}}$	&$\mathbf{82.09_{\pm0.17}}$	&$\mathbf{66.32_{\pm0.52}}$	&$\mathbf{67.72_{\pm0.85}}$	&$\mathbf{0.33_{\pm0.02}}$&$\underline{176.55_{\pm0.78}}$\\
	\midrule
	\midrule
	\multirow{9}{*}{Pair}	
	
&CVAE		&$48.54_{\pm1.96}$	                  &$93.45_{\pm2.23}$	                  	&$95.05_{\pm2.65}$	&$92.31_{\pm1.88}$	&$85.83_{\pm2.44}$	&$75.18_{\pm6.98}$	&$86.01_{\pm6.14}$	&$75.34_{\pm4.19}$	&$77.32_{\pm9.24}$	&$0.42_{\pm0.17}$&$187.88_{\pm6.09}$\\
&CEVAE		&$50.36_{\pm0.28}$	                  &$\mathbf{96.65_{\pm1.38}}$	&$\mathbf{95.40_{\pm1.98}}$	&$\underline{92.84_{\pm3.39}}$	&$\underline{88.98_{\pm4.94}}$	&$\underline{88.25_{\pm1.41}}$	&$\underline{93.82_{\pm1.07}}$	&$\underline{86.41_{\pm3.22}}$	&$\underline{85.57_{\pm3.61}}$	&$\underline{0.75_{\pm0.07}}$&$181.29_{\pm8.57}$\\
&mCEVAE	&$47.49_{\pm2.34}$	                  &$93.70_{\pm1.37}$	                  	&$94.65_{\pm3.10}$	&$92.14_{\pm1.15}$	&$84.65_{\pm2.15}$	&$65.75_{\pm1.95}$	&$87.07_{\pm3.48}$	&$77.90_{\pm8.54}$	&$78.73_{\pm4.90}$	&$0.43_{\pm0.21}$&$175.30_{\pm1.26}$\\
&DCEVAE	&$50.05_{\pm0.21}$	                  &$93.81_{\pm1.25}$	                  	&$\underline{95.14_{\pm1.06}}$	&$\mathbf{94.68_{\pm0.89}}$	&$\mathbf{93.89_{\pm0.73}}$	&$\mathbf{91.98_{\pm1.75}}$	&$\mathbf{95.13_{\pm0.72}}$	&$\mathbf{91.20_{\pm0.78}}$	&$\mathbf{88.77_{\pm1.20}}$	&$\mathbf{0.98_{\pm0.01}}$&$176.55_{\pm0.78}$\\
\cmidrule{2-13}
&CWGAN	&$\mathbf{55.92_{\pm1.57}}$	&$93.87_{\pm1.01}$	                  &$75.13_{\pm2.22}$	&$85.45_{\pm1.0}$	&$62.45_{\pm1.17}$	&$85.57_{\pm0.91}$	&$81.31_{\pm2.32}$	&$69.25_{\pm0.89}$	&$65.91_{\pm1.24}$	&$0.11_{\pm0.01}$&$\mathbf{106.14_{\pm1.16}}$\\
&DCGAN		&$51.16_{\pm0.37}$	                  &$81.50_{\pm5.69}$                        &$66.61_{\pm3.49}$	&$75.62_{\pm6.49}$	&$63.45_{\pm3.61}$	&$78.11_{\pm3.76}$	&$88.59_{\pm3.33}$	&$72.24_{\pm3.64}$	&$66.87_{\pm2.76}$	&$0.13_{\pm0.02}$&$141.94_{\pm3.60}$\\
&BEGAN		&$50.93_{\pm1.83}$	                  &$\underline{94.11_{\pm3.07}}$                        &$73.80_{\pm6.67}$	&$88.99_{\pm5.72}$	&$69.22_{\pm8.08}$	&$82.36_{\pm2.74}$	&$83.60_{\pm9.12}$	&$76.36_{\pm3.32}$	&$72.47_{\pm6.86}$	&$0.30_{\pm0.35}$&$192.11_{\pm93.46}$\\
&DCGAN-IC	&$50.93_{\pm0.97}$	                  &$80.23_{\pm3.25}$	                &$62.66_{\pm2.89}$	&$70.98_{\pm1.85}$	&$63.94_{\pm3.03}$	&$75.08_{\pm5.86}$	&$89.59_{\pm2.66}$	&$72.84_{\pm4.37}$	&$68.19_{\pm2.32}$	&$0.09_{\pm0.01}$&$\underline{139.80_{\pm2.72}}$\\
&BEGAN-IC	&$\underline{51.58_{\pm1.11}}$	                  &$89.91_{\pm1.74}$	                &$60.61_{\pm1.43}$	&$82.12_{\pm3.21}$	&$62.18_{\pm1.62}$	&$75.92_{\pm3.21}$	&$86.15_{\pm3.73}$	&$71.76_{\pm2.35}$	&$66.44_{\pm2.11}$	&$0.08_{\pm0.02}$&$145.9_{\pm2.89}$\\
	\bottomrule
 \end{tabular}}}
\end{center}
\caption{(1) Real: comparison between real images and generated counterfactual examples; Pair: comparison between reconstructed images and counterfactual images with the same latent values (2) Target Accuracy, (3) Identity Preserving (IP) scores (threshold$=0.2$ for Real, and $0.6$ for Pair), (4) The attribute classification accuracy for unaltered labels, (5) FID Score (FID). The numbers in bold indicate the best performance, and the underlined numbers indicate that the second best performance. (Here, M: \textit{Mustache}; WL: \textit{Wearing Lipstick}, MSO: \textit{Mouth Slightly Open}, S: \textit{Smiling}, ML: \textit{Male}, E: \textit{Eyeglasses}, B: \textit{Bald}, NE: \textit{Narrow Eyes}, and Y: \textit{Young}).}
\label{table:Eval_CelebA_M}
\end{table*}

\subsection{Image Generation Task}
\noindent We choose \textit{Mustache} and \textit{Smiling} as intervention variables in the CelebA dataset. Our model and baselines generate the counterfactual image by negating the value of the intervention variable. For example, if the real image has $\textit{Mustache}=0$, the counterfactual image should have $\textit{Mustache}=1$.

This section describes (1) the visualization of counterfactual images from our model and baselines; (2) the analysis of the latent variables from each model; and (3) the quantitative analysis for the image generation task.
\subsubsection{Generated Counterfactual Images}
\noindent Figure \ref{fig:CF_img} visualizes the counterfactual images from the VAE based models for the real image. VAE based models can infer the exogenous variables from a specific real image. Counterfactual images from DCEVAE preserve the identity of real images except for the intervention variables. For example, female images on the first row in Figure \ref{fig:CF_img} are kept to be female while the intervened \textit{Mustache} ($M$) is added. On the contrary, CVAE makes images with \textit{Mustache}, but its gender is altered. On the other hand, CEVAE creates blurry images because CEVAE separates the decoders by the case of interventions, so a female image is rarely given to training the decoder of $M=1$. 
mCEVAE also fails in generating images because MMD in mCEVAE enforces the latent variables of $M=0$ and $M=1$ overlapped. The distribution of $M=1$ and the joint distribution of $M=0$ and $\textit{Male}=1$ have a large overlapping area, so mCEVAE can make counterfactual images with respect to $M=1$, but not for all images with $M=0$ and $\textit{Male}=0$, see the distribution in Figure \ref{fig:CF_img}. It should be noted that we did not include the GAN-based counterfactual generation because the GAN variations cannot produce the counterfactual image matched to a given real-world image.
\begin{figure}[h!]
\center
\includegraphics[width=0.15\columnwidth]{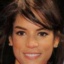}
\includegraphics[width=0.15\columnwidth]{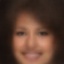}
\includegraphics[width=0.15\columnwidth]{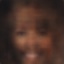}
\includegraphics[width=0.15\columnwidth]{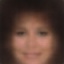}
\includegraphics[width=0.15\columnwidth]{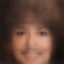}\\
\vspace{0.5mm}

\includegraphics[width=0.15\columnwidth]{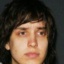}
\includegraphics[width=0.15\columnwidth]{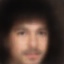}
\includegraphics[width=0.15\columnwidth]{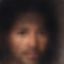}
\includegraphics[width=0.15\columnwidth]{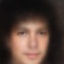}
\includegraphics[width=0.15\columnwidth]{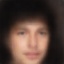}\\
\vspace{0.5mm}

\includegraphics[width=0.15\columnwidth]{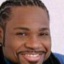}
\includegraphics[width=0.15\columnwidth]{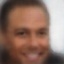}
\includegraphics[width=0.15\columnwidth]{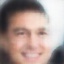}
\includegraphics[width=0.15\columnwidth]{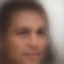}
\includegraphics[width=0.15\columnwidth]{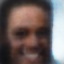}\\

\includegraphics[width=0.18\columnwidth]{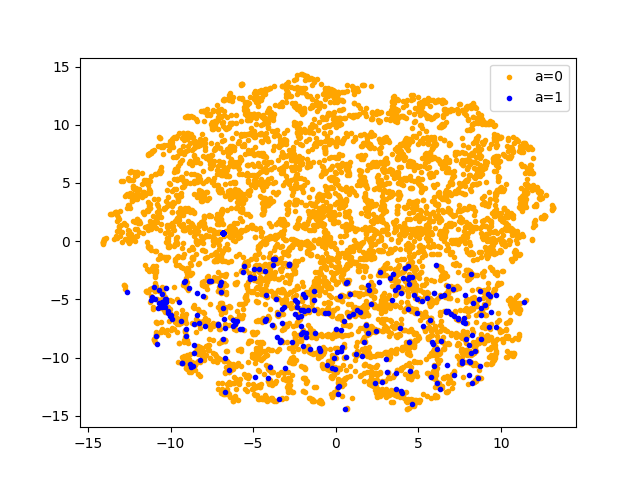}
\includegraphics[width=0.18\columnwidth]{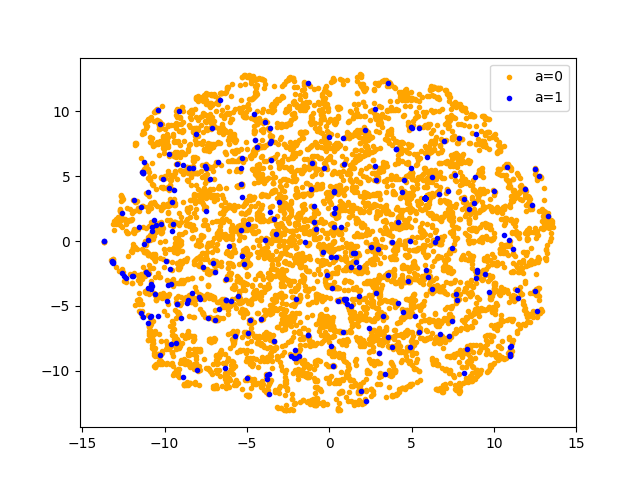}
\includegraphics[width=0.18\columnwidth]{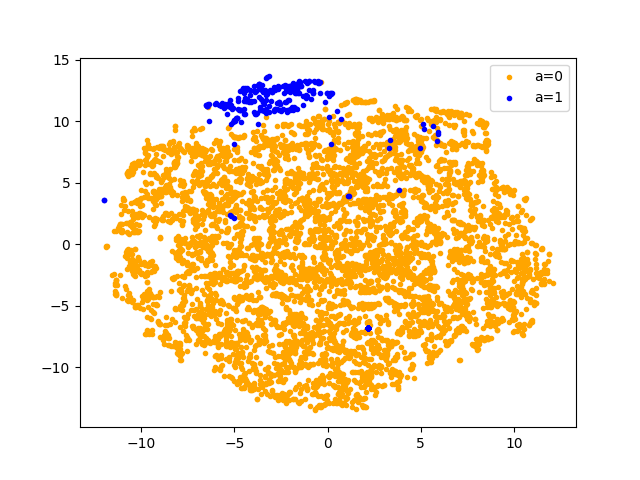}
\includegraphics[width=0.18\columnwidth]{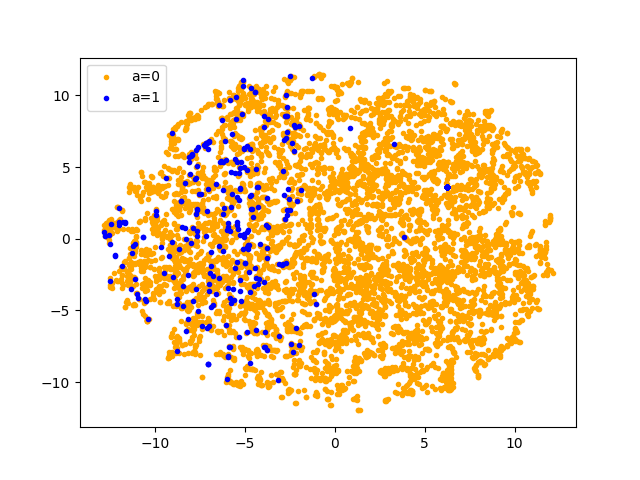}
\includegraphics[width=0.18\columnwidth]{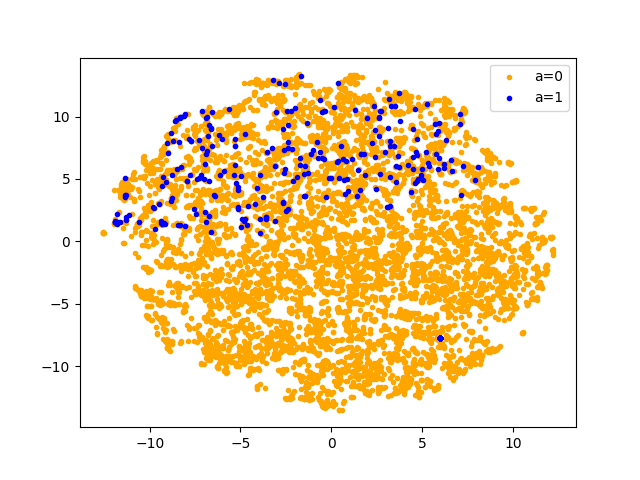}
\caption{(rows 1-3) Real image and its counterfactual images for \textit{Mustache} from DCEVAE, CVAE, CEVAE, and mCEVAE. (row 4) tSNE \cite{maaten2008visualizing} of latent variables $u_r$, $u_d$ in each model. Here, $M=0$ (yellow) and $M=1$ (blue).}
\label{fig:CF_img}
\end{figure}
\subsubsection{Quantitative Analysis on Image Generation Task}
\label{sec:quant}
\noindent This section shows the quantitative evaluations on generated counterfactual images with label classifier accuracies, Frechet Inception Distance (FID) score, and identity preserving (IP) metrics. Table \ref{table:Eval_CelebA_M} shows the result of counterfactual generation on \textit{Mustache}.
We compare the real and the counterfactual generated images of VAE-based approaches. DCEVAE has the highest attribute classification accuracies, so DCEVAE maintains the attribute of $x_r$ intact. Also, DCEVAE has the highest IP and the lowest FID scores, so the generated images are evaluated to be more natural than the other models. The generation of $Mustache$ is also measured by a classifier accuracy, and CEVAE, mCEVAE, and DCEVAE have similar accuracies. We compare reconstructed images and counterfactual images from VAEs and GANs. Except FID score, DCEVAE preserves attributes causing \textit{Mustache}, i.e. preserved \textit{Male} attribute.
\section{Conclusion}
This paper disentangles the exogenous uncertainty into two latent variables of 1) independent to interventions ($u_d$), and 2) correlated to interventions without causality ($u_r$). The disentanglement of latent variables resolves the limitation in previous works, including maintaining causality from the intervention ($a$) and altering all correlated information for counterfactual instances. Our model, DCEVAE, estimates the total effect and the counterfactual effect without a complete causal graph. In experiments, we showed that DCEVAE is comparable with other models with and without the complete causal graph. Both applications on the fair classification and the counterfactual generation showed the best quantitative performance by utilizing a counterfactual instance matched to the real-world instances.

\section{Ethical Impact}
Besides the COMPAS incident \cite{brennan2009evaluating}, governments and corporates utilize the AI-based screening and recommendation systems on a massive scale, and these applications are prone to the fairness question, particularly when the subject individual has minority backgrounds. This paper discusses the triad of 1) the fair classification, 2) the causality-based counterfactual generation, and 3) the latent disentanglement. If we were to maximize the classification accuracy, the proposed method would be irrelevant from such efficiency-oriented perspectives. However, our society always asks what-if questions, i.e., the veil of ignorance by Rawls \cite{rawls2009theory}. The limitation of the accuracy can be acceptable in two conditions: 1) the damage to the accuracy performance should be controlled and minimal, and 2) the limitation satisfies the argument that "I would accept the classification result under my altered background." This "altered background" is, in fact, an identical argument to the justice concept suggested by Rawls, which argues designing a taxation concept before determine whether you will be born in either high-income or low-income families. An emerging question is whether or not an AI can come up with a justifiable altered concept on what-if scenarios, so we work on the counterfactual generation to satisfy the fairness concept defined in the above. This counterfactual generation is further elaborated if an AI carefully dissect the context of an individual subject, and this is the disentanglement process when causality and a correlation should be distinguished. Given this series of arguments and necessity, this work is an important contribution in promoting the fairness of deployed AI systems, which are already running without a user's perception of its background operation. 

\section{Acknowledgement}
This research was supported by Basic Science Research Program through the National Research Foundation of Korea (NRF) funded by the Ministry of Education(NRF-2018R1C1B600865213)

\bibliography{reference.bib}

\end{document}